\documentclass[preprint,authoryear,12pt]{elsarticle}

\usepackage{algorithmic}
\usepackage{algorithm}
\usepackage{amssymb}
\usepackage{placeins}
\usepackage{color}

\newtheorem{definition}{Definition}
\newtheorem{lemma}{Lemma}
\newtheorem{proof}{Proof}
\newtheorem{theorem}{Theorem}
\newtheorem{corollary}{Corollary}

\journal{--}

\begin{document}

\begin{frontmatter}

\title{The Partner Units Configuration Problem: \\ Completing the Picture}
\author{Erich Christian Teppan\footnote{Corresponding author. contact: erich.teppan@aau.at, +43 463 2700 3756} and Gerhard Friedrich}
\address{Universitaet Klagenfurt, Austria}

\begin{abstract}

The partner units problem (PUP) is an acknowledged hard benchmark problem for the Logic Programming community with various industrial application fields like surveillance, electrical engineering, computer networks or railway safety systems. However, computational complexity remained widely unclear so far. In this paper we provide all missing complexity results making the PUP better exploitable for benchmark testing. Furthermore, we present QuickPup, a heuristic search algorithm for PUP instances which outperforms all state-of-the-art solving approaches and which is already in use in real world industrial configuration environments.
\end{abstract}


\end{frontmatter}

\section{Introduction}

The partner units problem (PUP, \cite{falkner:modeling}) is a classical configuration problem where components have to be connected such that all user requirements and technical constraints are satisfied (see~\cite{mittal}). Solving such real world configuration problems are one of the major success stories of Artificial Intelligence which resulted in a commercially attractive area where many companies are offering configuration tools and services. Current modern configuration tools apply declarative knowledge representation and reasoning techniques based on constraint satisfaction, SAT solving, Answer Set Programming or Description Logics (see~\cite{junker} for an overview). Consequently, there is a strong interest that these techniques are applicable for typical configuration problems. 

Given the results of the ASP competition 2011\footnote{Summarized results and raw data set available at https://www.mat.unical.it/aspcomp2011/Model\%26SolveTrackFinalResults} (\cite{aspcomp}) and the evaluation documented in~\cite{aschinger:opt} the PUP is an exceptionally hard real world problem. Although general problem solving techniques based on constraint programming, SAT solving or Answer Set Programming are applicable for small and some mid-sized PUP instances, large real world cases are clearly out of reach for current solver technologies.

Figure \ref{fig:aspcomp} shows the 34 benchmark problems in the model-and-solve track of the ASP competition 2011 and depicts how often corresponding instances could be consistently solved or proved to be unsatisfiable within time and memory limits\footnote{Detailed information is available at \textit{https://www.mat.unical.it/aspcomp2011/Participants}}. Obviously, the hardest problem was the strategic company problem (\cite{cadoli:1997}, \cite{leone:2006}) which is complete on the second level of the polynomial hierarchy, i.e.\ $\Sigma^P_2 = NP^{NP}$. 

As it turns out the second hardest problem, and thus the hardest problem in NP out of all 2011 ASP competition model-and-solve benchmark problems, was a subclass of the PUP which had already been proven to be of polynomial complexity (\cite{aschinger:tackling}). Also a second PUP benchmark for which no complexity results were available so far, belonged to the hardest third of tested benchmark problems. Consequently, the PUP is both from the practical and theoretical point of view a very important problem for further investigations.

\begin{figure}
\center
	\includegraphics[width=14cm]{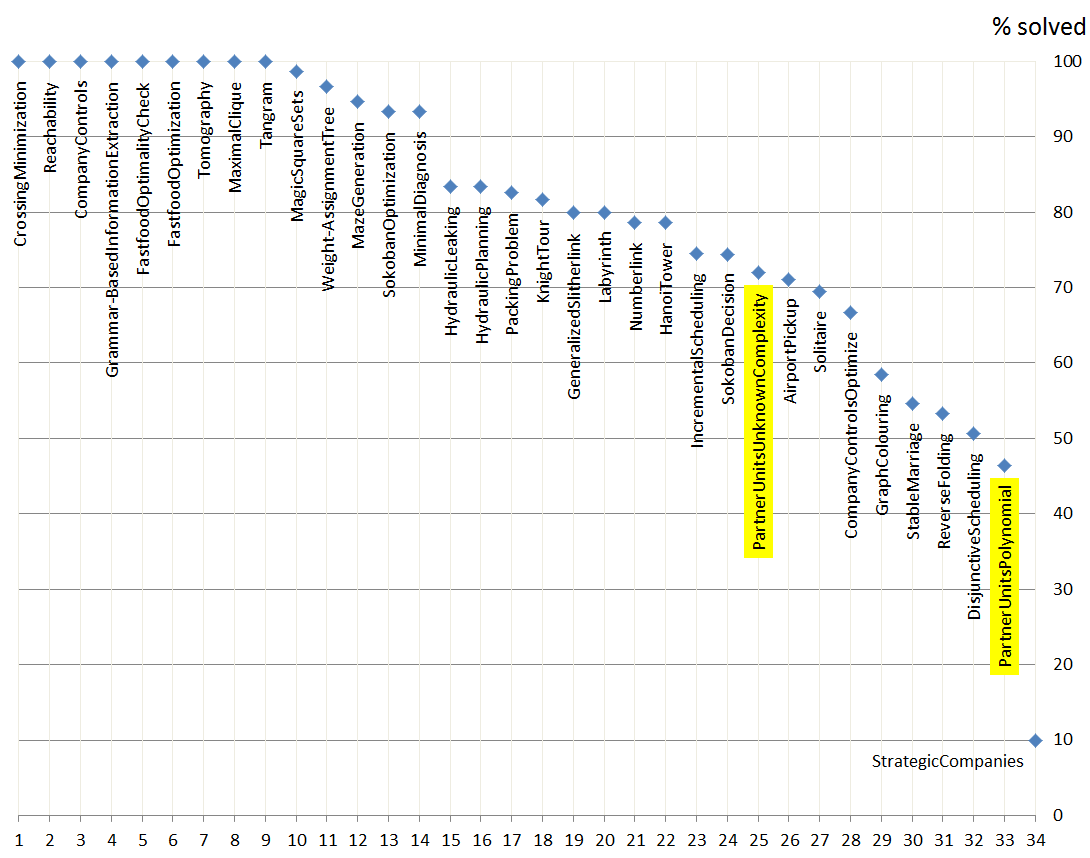}
	\caption{Percentage of ASP competition runs where a solution could be found by a solver or unsatisfiability could be proved}
	\label{fig:aspcomp}
\end{figure}

Originally, the partner units problem was identified in the domain of railway safety systems. One of the problems in this domain is to make sure that certain rail tracks are not occupied by a train/wagon before another train enters this track.
For deciding if a rail track is occupied, occupancy indicators and wheel sensors for counting the number of train wheels passing a wheel sensor are connected to processing units. Because of fail-safety and realtime requirements the number of sensors and indicators which can be connected to the same unit is limited (called unit capacity, UCAP). Also one sensor/indicator can only be directly connected to one unit. Moreover, a unit can only be connected to a limited number of other units (called inter unit capacity, IUCAP). These units are called the partner units of the unit. Elements can only communicate with elements connected to the same unit and with elements connected to one of the partner units. Given an input graph specifying which sensor data is needed in order to calculate the correct signal of an occupancy indicator, the problem consists in connecting sensors/indicators with units and units with other units such that all communication requirements are fulfilled.

\begin{figure}
\center
	\includegraphics[width=13cm]{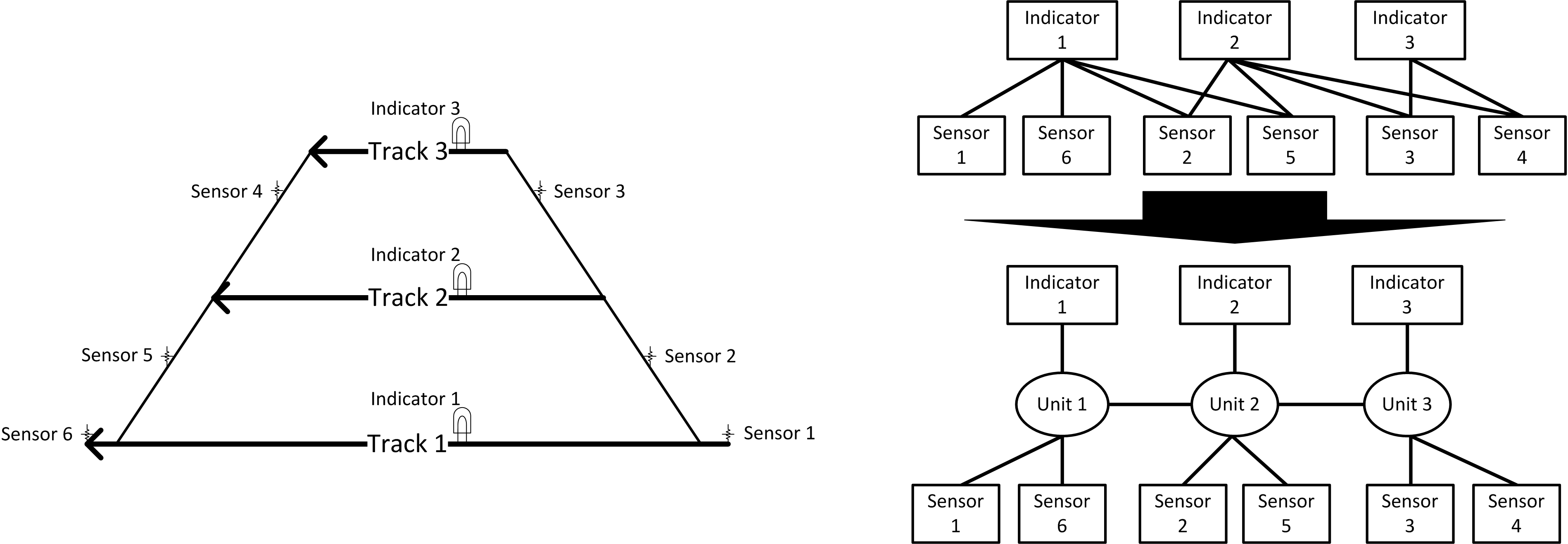}
	\caption{Example of a railway track layout}
	\label{fig:example1}
\end{figure}

Figure \ref{fig:example1} shows a simple example with three connected railway tracks with corresponding occupancy indicators and wheel sensors, the corresponding PUP input graph and a solution wherby UCAP = 2, i.e.\ only two sensors and two indicators can be connected to the same unit, and IUCAP = 2, i.e.\ each unit can at most have two partner units. In order to calculate the correct signal for Indicator 3 only data from Sensor 3 and Sensor 4 is needed. If the number of outgoing wheels counted by Sensor 4 is equal to the incoming wheel counts of Sensor 3 then Track 3 is empty. For Track 2 and Track 3 it is somewhat more complex. In order to calculate the correct signal for Indicator 2 it is not sufficient to only incorporate data from Sensor 2 and Sensor 5 as it is not clear whether a wheel has headed to or is coming from Track 3. Therefore, additional data from Sensor 3 and Sensor 4 is needed. For Indicator 1 data from Sensor 1, Sensor 2, Sensor 5 and Sensor 6 must be considered.

Further important application domains apart from railway safety are electrical engineering, peer-to-peer networking and CCTV surveillance (\cite{teppan:reconf}, \cite{aschinger:opt}).




At tis point, there can be drawn three important conclusions:

\begin{enumerate}
	\item The PUP is a very hard real world problem which is perfectly suitable for benchmark testing in the declarative programming community.
	\item In order to further and better exploit the PUP as a benchmark problem, complexity results for all problem subclasses are needed.
	\item Because of the practical relevance of the PUP for many industrial fields and also as a benchmark problem, it is further desirable to have a competitive solving strategy in order to:

\begin{itemize}
	\item provide solutions for large real world cases.
	\item provide a yardstick for benchmark testing of general problem solving techniques. 
\end{itemize}
\end{enumerate}

Our main contributions are the following: 

\begin{enumerate}
	\item We deliver complexity results for all PUP sub problems and show that except all PUP subclasses which were not proved polynomial so far are NP-complete.
	\item We present a novel heuristic backtracking algorithm which performs better than state-of-the-art approaches by orders of magnitude and solves real world PUP instances in milliseconds. 
\end{enumerate}

The paper is structured as follows. In Section 2 we provide the complexity results for all possible subclasses. QuickPup, our heuristic backtracking search algorithm, is introduced in Section 3. In Section 4 we provide experimental results which are based on real-world problem instances.

\section{Complexity of the PUP}

With regard to the origin of the PUP we refer to
the two types of elements to be placed on communication
units as \textit{indicators} and \textit{sensors} and to the communication units
as \textit{units} for the rest of the paper. Formally, the PUP can be defined as follows:

\begin{definition}
A partner units problem instance is a sixtuple $P = <I,S,E,U,\\UCAP,IUCAP>$ where $I$ represents a set of indicators, $S$ represents a set of sensors, $U$ is a set of units and $E \subseteq I \times S$ is the set of edges between $I$ and $S$ in the corresponding bipartite input graph. Given the two natural numbers $UCAP$ (unit capacity) and $IUCAP$ (inter unit capacity), the PUP decision problem consists in deciding whether there exists a solution function $f: I \cup S \rightarrow U$ such that:

\begin{itemize}
\item For every $u \in U$:\\
$I_{u} = \{i | i \in I \wedge f(i)=u\}, |I_{u}| \leq UCAP$,\\ 
$S_{u} = \{s | s \in S \wedge f(s)=u\}, |S_{u}| \leq UCAP$
\item Every $u,v \in U$ with $u \neq v$ are connected, whenever 
$i \in I_{u} \wedge s \in S_{v} \wedge (i,s) \in E$

\item The connection relation is symmetric, i.e.\ if $u \in U$ is connected to $v \in U$ then $v$ is connected to $u$
 
\item Every unit $u \in U$ is connected to at most $IUCAP$ other units.
\end{itemize}
\end{definition}

The solution function corresponds to a solution graph showing which indicators and sensors are placed on which units and how the units are connected to each other (see Figure \ref{fig:example1}).

\begin{figure}
\center
	\includegraphics[width=8cm]{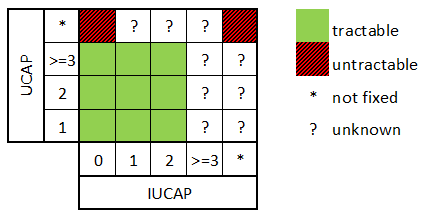}
	\caption{Current complexity landscape for the PUP}
	\label{fig:preoverview}
\end{figure}

Figure \ref{fig:preoverview} shows an overview about the complexity landscape of the partner units problem (PUP). In the most general form, that is when IUCAP and UCAP are part of the input, the PUP is NP-complete. This was shown by reducing bin-packing to the special case where IUCAP = 0, thus when the units in fact constitute bins (\cite{aschinger:tackling}, \cite{garey:np}). Also the special case with IUCAP = 1 is basically the same as with IUCAP = 0. This is due to the fact that cases with IUCAP = 0 can  be transformed to IUCAP = 1.

\begin{corollary}
The PUP is NP-complete when IUCAP = 1 and UCAP is part of the input.
\end{corollary}

\begin{proof}
Membership in NP is evident. For showing completeness, we can transform PUP with IUCAP = 0 to PUP with IUCAP = 1 as follows:

\begin{enumerate}
\item For each unit in the PUP with IUCAP = 0 there are two units for IUCAP = 1.
\item For every indicator $i$ in the input graph for we add a dummy indicator $d_i$ and for every sensor $s$ in the input graph we add a dummy sensor $d_s$.
\item For connected elements $i$ and $s$ we add also an edge between: 
\begin{itemize}
	\item $d_i$ and $d_s$
	\item $i$ and $d_s$
	\item $s$ and $d_i$
\end{itemize}
\end{enumerate}

The new input graphs for IUCAP = 1 have exactly double the size as the original input graphs for IUCAP = 0 but there are also double the number of units available, which can be connected to pairs of units. Thus, a PUP instance with IUCAP = 0 is satisfiable if and only if the corresponding PUP instance with IUCAP = 1 is satisfiable.$\square$ 
\end{proof} 

For the cases with IUCAP = 2 and fixed UCAP a polynomial-time algorithm could be found due to the fact the solution input graphs for such cases are always chains or rings of units \cite{aschinger:tackling}. For the cases with fixed IUCAP $\geq$ 3 complexity remained unclear so far. Also for the complexity with IUCAP = 2 but not fixed UCAP there has not been a complexity result. For practical purposes the fixed parameter complexity when both parameters, i.e.\ IUCAP and UCAP, are fixed to some natural number is of special importance. This is because typically the PUP occurs in application cases where standardized devices (indicators/sensors, communication, units circuit boards, etc.) are used which provide a fixed number of connectors and ports. 

In the following it will be shown that any PUP with unfixed UCAP or with fixed IUCAP $\geq$ 3 is NP-complete. Thus, we will be able to color all white gaps on the complexity landscape in Figure \ref{fig:preoverview}. 

\begin{figure}
\center
	\includegraphics[width=11cm]{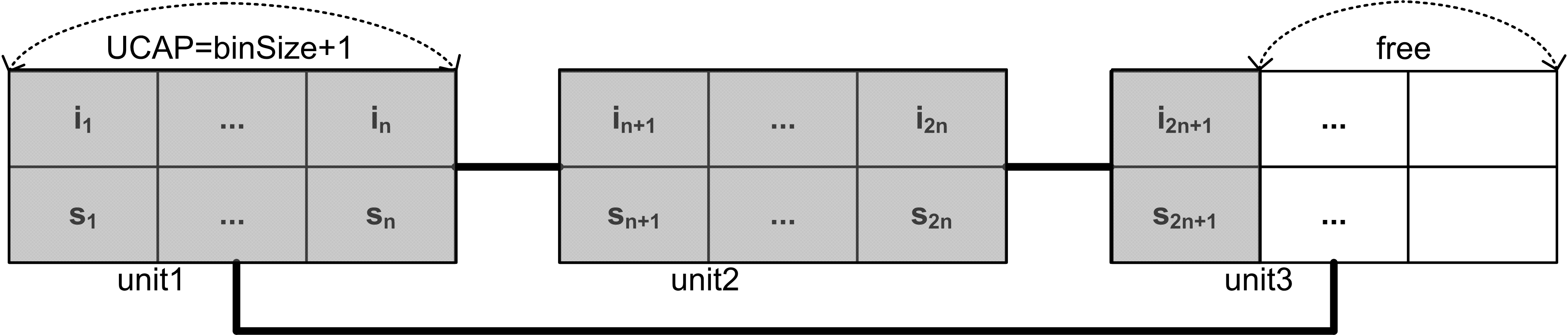}
	\caption{Principle of bins for PUP with IUCAP = 2 and UCAP not fixed}
	\label{fig:bew2}
\end{figure}

\begin{theorem}
The PUP is NP-complete when IUCAP=2 and UCAP is part of the input and thus not fixed.
\end{theorem}

\begin{proof}
Membership in NP is evident. For NP-completeness we reduce from bin-packing given by a set of natural numbers $N=\{n_1 \dots n_n\}$ representing the item sizes, the bin size $binSize$ and the number of bins $k$. The PUP instance is produced as follows:

\begin{itemize}
\item Set $UCAP=binSize+1$ and $IUCAP=2$
\item For every item of size $n \in N$ produce a new indicator $i$ and $n$ bicliques $(i,si_{1}) \dots (i,si_{n})$ with $si_{1}...si_{n}$ constituting fresh sensors.
\item For each bin introduce $(2 \times UCAP) + 1$ fresh indicators and $(2 \times UCAP) + 1$ fresh sensors and produce bicliques between all those indicators and sensors. As every indicator and sensor is connected to $(2 \times UCAP) + 1$ elements, every unit has exactly 2 partner units. Furthermore, as all elements share the same elements they also share the same partner units which results in a ring structure consisting of 3 units (see Figure \ref{fig:bew2}). For every such structure there remain $UCAP-1 = binSize$ free slots for placing the item structures. Note that the free slots may be arbitrarily distributed on the units.
\end{itemize}

There exists a packing with $k$ or fewer bins iff there exists a solution to the PUP with $3 \times k$ units.$\square$

\end{proof}

For proving NP-completeness results for PUP instances with IUCAP $\geq$ 3, we use a special form of bin-packing.

\begin{lemma}
\label{lem1}
Bin packing is NP-complete when item sizes and bin sizes are even.
\end{lemma}

\begin{proof}
Given a regular bin packing problem given by a set of natural numbers $N=\{n_1 \dots n_n\}$ representing the item sizes, the bin size $b$ and the number of bins $k$,
we can transform regular bin packing by multiplying item sizes $n \in N$ and the bin size $b$ with two. Every combination of items which summed up to $\leq b$ in the regular problem sums up to $\leq 2 \times k$ in the transformed problem. Analogously, every combination of items which summed up to $> b$ in the regular problem sums up to $> 2 \times k$ in the transformed problem.$\square$
\end{proof}

Before we introduce the structures which are necessary for transforming bin packing to the PUP, we need the notion of induced input graphs.
Informally, an induced input graph is an input graph that is produced based on a given solution graph (which we call inducing graph) such that there is an edge between an indicator and a sensor in the induced input graph whenever they can communicate in the given solution graph (inducing graph).

\begin{definition}
Given a PUP solution graph $G = <I \cup S,U,V,W>$ which we call \emph{inducing graph}, with the indicators and sensors $I \cup S$, the units $U$, the edges $V$ between $I \cup S$ and $U$, and the edges $W$ between units in $U$, the induced input graph is the bipartite input graph $P = <I,S,E>$ such that $(i,s) \in E$ with $i \in I$ and $s \in S$ iff

\begin{itemize}
\item $(i,u) \in V \wedge (s,u) \in V$ with $u \in U$ or
\item $(i,u1) \in V \wedge (s,u2) \in V$ and $(u1,u2) \in W$ with $u1,u2 \in U$
\end{itemize}
\end{definition}

In order to simulate even sized bins (see Lemma \ref{lem1}) by means of PUP structures we need the graphs defined in Figure \ref{fig:binslots} employing UCAP$=1$ and IUCAP$=3$. Two graphs are to be used in combination. The Figure shows the inducing graphs and the induced input graphs for both. Furthermore, Figure \ref{fig:binslots} introduces the graph condensing symbols which are to be further used for ease of presentation. In the subsequent proofs we exploit the following property of the solution graph given the input graph of Figure \ref{fig:binslots}.

\begin{lemma}
\label{lem2}
The PUP input graphs in Figure \ref{fig:binslots} guarantee that indicator 'a' is placed on a unit that contains a free sensor slot at this stage and which contains two free unit connections. Likewise for the other graph, it is guaranteed that sensor 'b' is placed in a unit that contains a free indicator slot and which contains two free unit connections.
\end{lemma}

\begin{proof}
\label{prooflem2}
For the ease of understanding the major steps of the proof are visualized by Figure \ref{fig:proof-lemma}.

\begin{enumerate}
\item $i1$ must communicate to all sensors in $\{s1,s2,s3,s4\}$ and as a consequence the unit hosting $i1$ must be connected to three other units and each of the four units must contain one of the sensors in $\{s1,s2,s3,s4\}$. We call the unit hosting $i1$ $U1$.
\item As $i2$ is also connected to all sensors in $\{s1,s2,s3,s4\}$ and $U1$ will also host one of the sensors in $\{s1,s2,s3,s4\}$, $i2$ must be placed on one of the three partner units of $U1$. We call the unit hosting $i2$ $U2$. The other two units have to be connected to $U2$. We call these units $U3$ and $U4$, depending on which element ($i3$ or $i4$ is placed on the unit).
\item $i3$ respectively $i4$ must be placed on $U3$ and $U4$ because $i3$ as well as $i4$ have to communicate to three sensors in $\{s1,s2,s3,s4\}$. As both $U1$ and $U2$ do not allow any more unit connections, $i3$ as well as $i4$ have to be on an already connected partner unit, i.e. $U3$ or $U4$.
\item $s1$ as well as $s2$ must go to $U1$ or $U2$ and not to $U3$ or $U4$ as $s1$ and $s2$ have to communicate to all four indicators in $\{i1,i2,i3,i4\}$. Placing $s1$ or $s2$ on $U3$ or $U4$ would demand that $U3$ gets connected to $U4$, making it impossible to connect any further elements.
\item Similarly, $s3$ must be placed on $U3$ (together with $i3$) and $s4$ must go to $U4$ (together with $s4$) and not the other way round. Placing $s3$ on $U4$ and $s4$ on $U3$ would again require that $U3$ and $U4$ are connected, making it impossible to connect any further elements.
\item As $U3$ as well as $U4$ have only one free unit connection left at this stage, and because both $i3$ and $i4$ have to communicate to $s5$ and furthermore because both $s3$ and $s4$ have to communicate to $i5$ the only possibility is to place $i5$ and $s5$ on the same unit ($U5$) which is connected to $U3$ and $U4$. As a consequence $U5$ has only one free unit connection left which is to be used for connecting the 'a' or 'b' element.$\square$
\end{enumerate}
\end{proof}

\begin{figure}
\center
	\includegraphics[width=13cm]{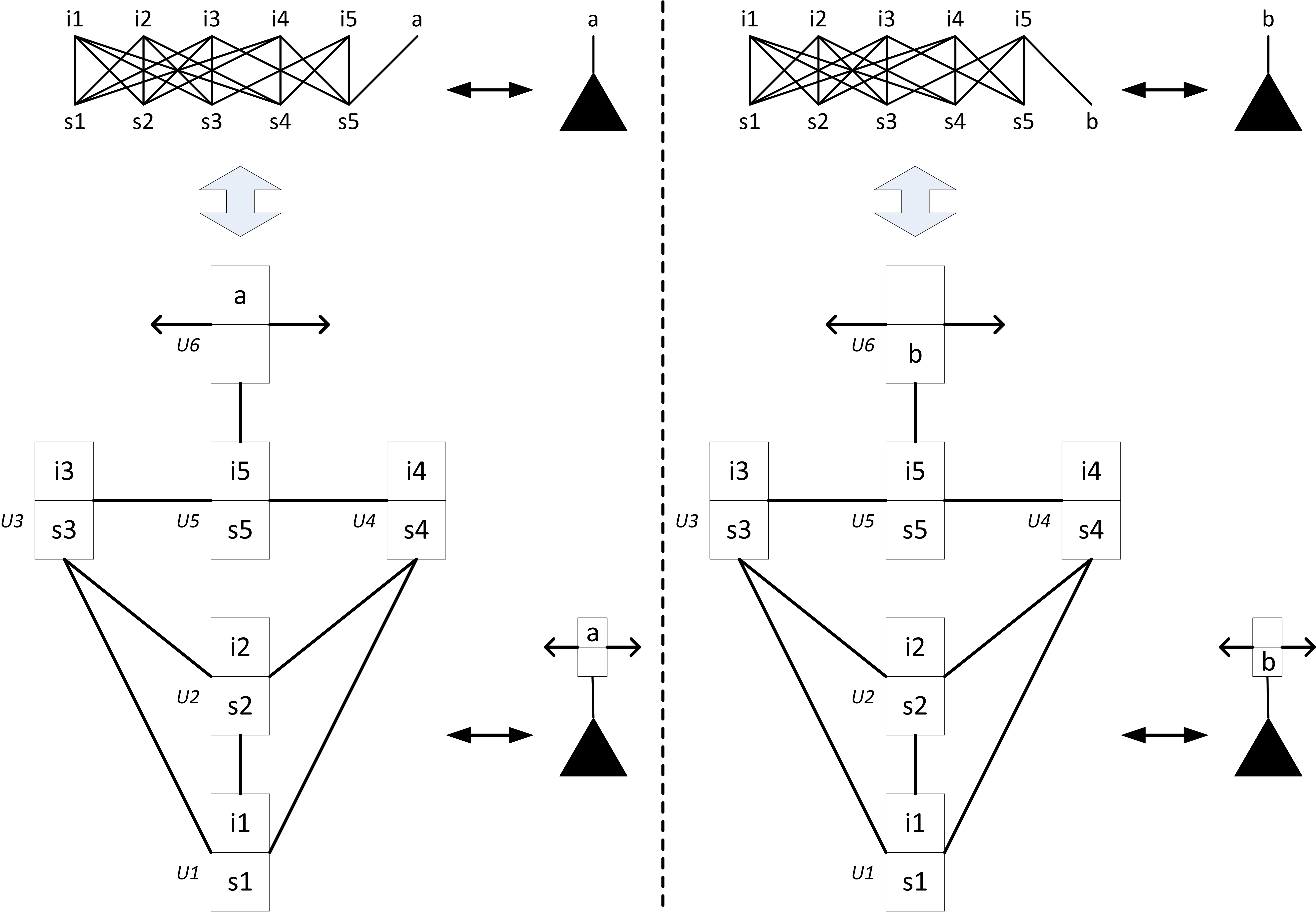}
	\caption{Inducing graphs and corresponding input graphs used for simulating bins}
	\label{fig:binslots}
\end{figure}

\begin{figure}
\center
	\includegraphics[width=13cm]{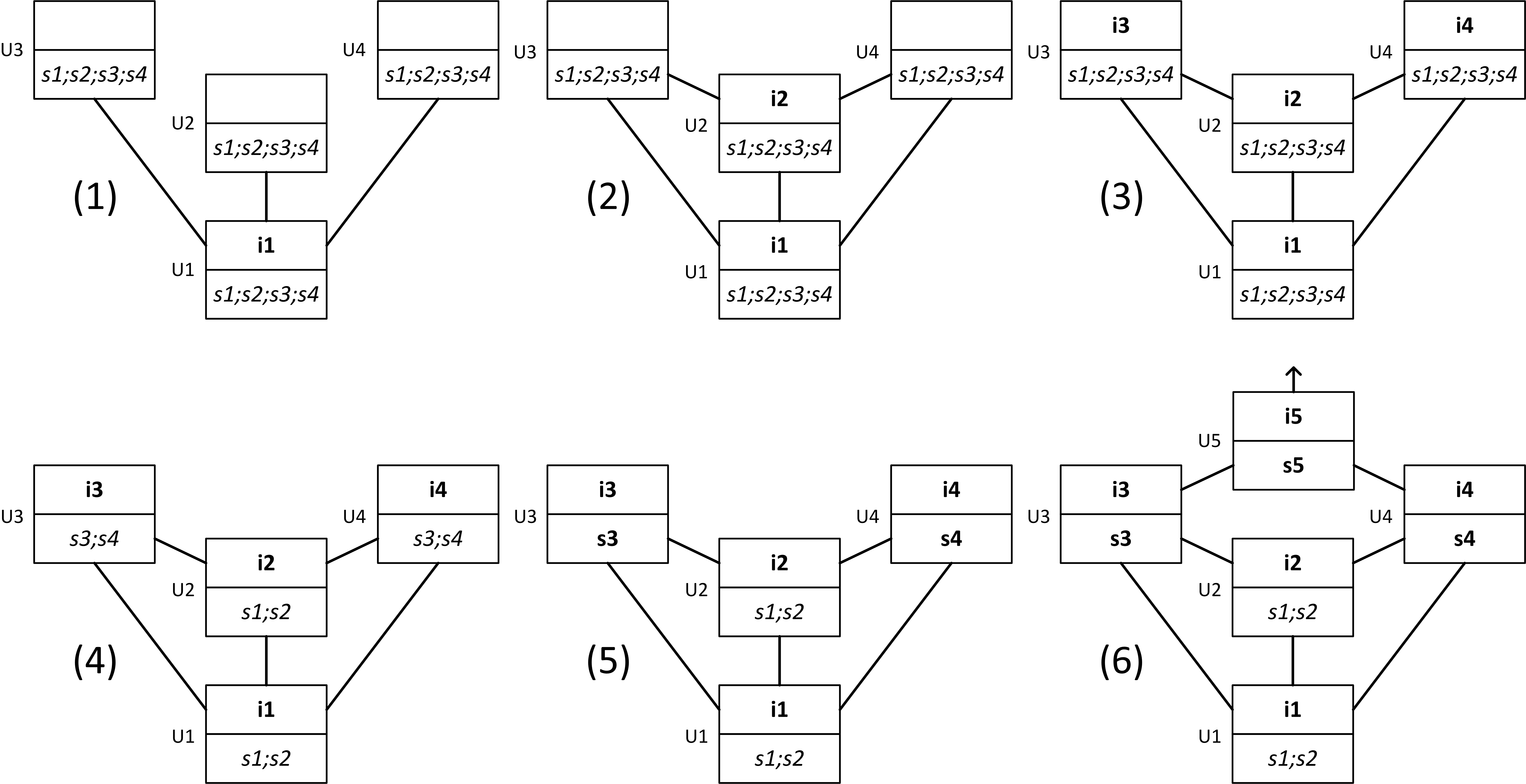}
	\caption{Proof of Lemma \ref{lem2}. ';' stands for a logical 'or'.}
	\label{fig:proof-lemma}
\end{figure}

Building on the graphs in Figure \ref{fig:binslots}, the main principle of simulating any even-sized bin is shown in Figure \ref{fig:bin}. The left-over unit slots are positioned in such a way that there is enough connected space to host additional indicator-sensor chains of maximal length equal to bin size, hence items are to be represented by chains of indicators and sensors. Figure \ref{fig:filledbin} gives a small example with one bin of size = 4 and one item of size = 4.

\begin{figure}
\center
	\includegraphics[width=13cm]{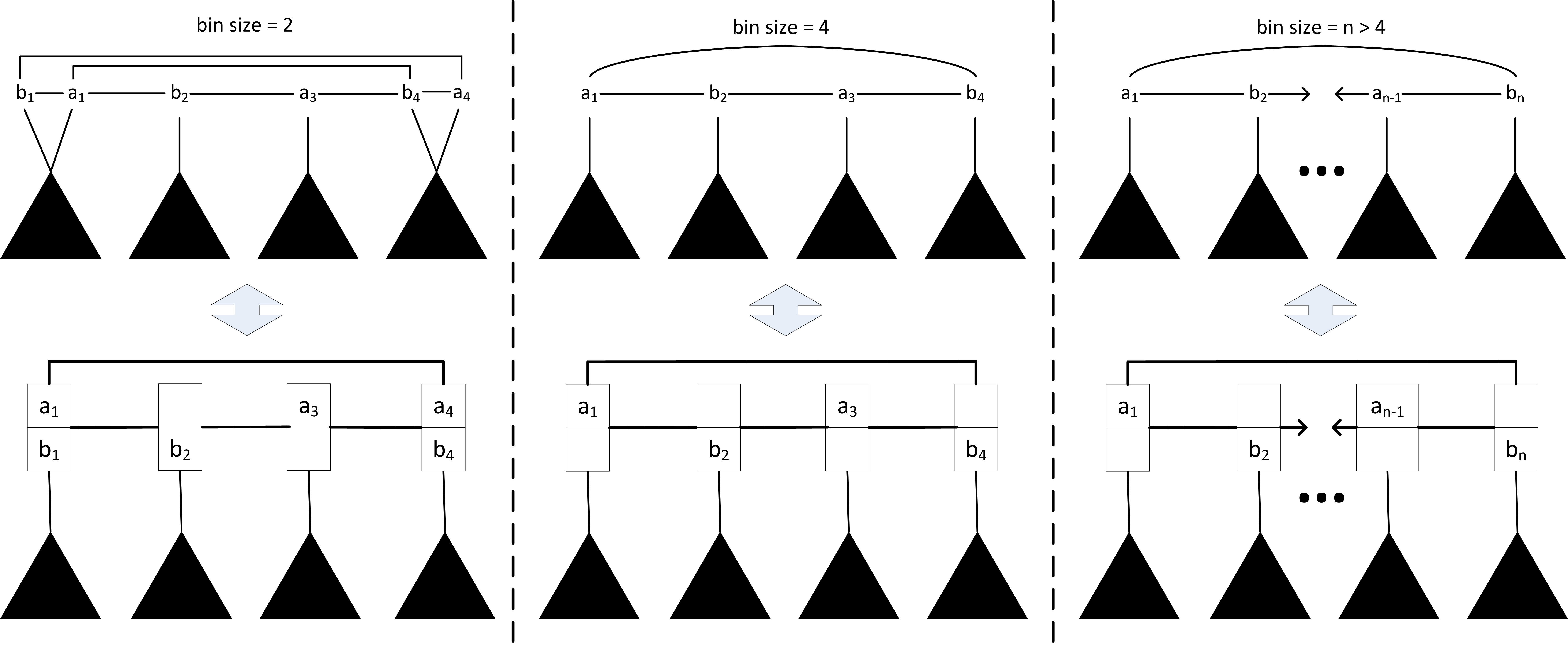}
	\caption{Simulating bins on the basis of the graphs given in Figure \ref{fig:binslots}}
	\label{fig:bin}
\end{figure}

\begin{figure}
\center
	\includegraphics[width=13cm]{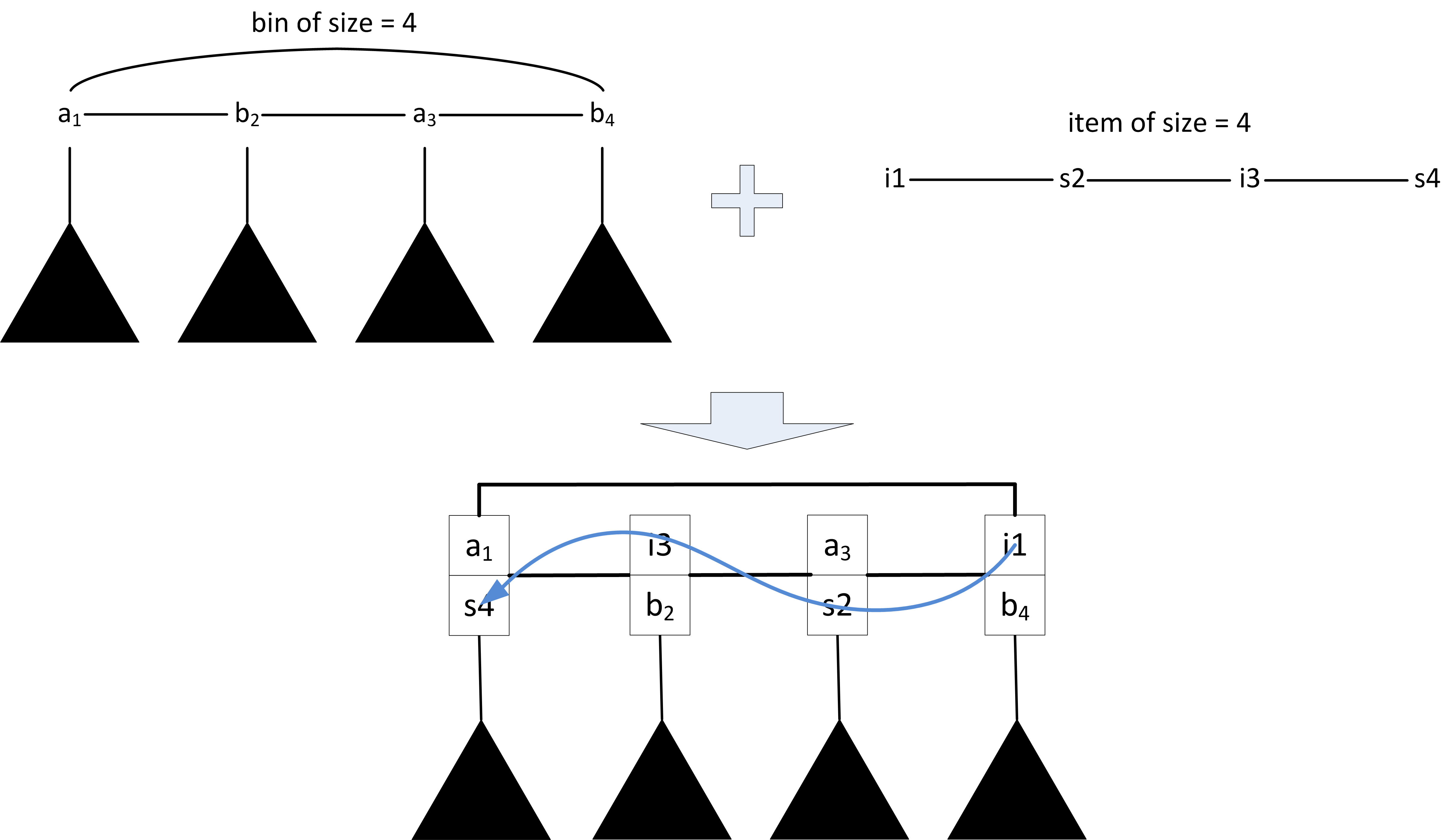}
	\caption{An item fitting in a bin}
	\label{fig:filledbin}
\end{figure}

\begin{theorem}
\label{theorem2}
The PUP is NP-complete when IUCAP = 3 and UCAP = 1.
\end{theorem}

\begin{proof}
\label{prooftheorem2}
Membership in NP is evident. For completeness we reduce from bin packing given by a set of natural numbers $N=\{n_1 \dots n_n\}$ representing the items and corresponding (even) item sizes, the (even) bin size $binSize$ and the number of bins $k$. The PUP instance is produced as follows:

\begin{enumerate}
\item For every bin introduce a new structure based on the graphs given in Figure \ref{fig:binslots} and Figure \ref{fig:bin}.
\item For every item size $n \in N$ introduce a set of fresh indicators $i_{n_{1}} i_{n_{n/2}}$ and a set of fresh sensors $s_{n_{1}} \dots s_{n_{n/2}}$ and connect them such that they build a chain, i.e.\ $i_{n_{1}}$ is connected to $s_{n_{1}}$, $s_{n_{1}}$ is connected to $i_{n_{2}}$, $i_{n_{2}}$ is connected to $s_{n_{2}}$, \dots, $i_{n_{n/2}}$ is connected to $s_{n_{n/2}}$.
\end{enumerate}

As there are only bins and items of even size each chain begins with an indicator and ends with a sensor or vice-versa. This facilitates avoiding gaps when filling the bins.
When $binSize = 2$, the bin packing instance is solvable with $k$ bins iff the corresponding PUP is solvable with $(6 \times 4) \times k$ units. When $binSize \geq 4$, the bin packing instance is solvable with $k$ bins iff the corresponding PUP is solvable with a set of $(6 \times binSize) \times k$ units.$\square$
\end{proof}

In order to generalize the proof to UCAP $>$ 1 we only have to show how bins are represented for such cases, i.e.\ it suffices to define corresponding inducing graphs which induce input graphs simulating bins.

\begin{corollary}
\label{corollary2}
The PUP is NP-complete when UCAP fixed to some $m \geq 2$ and IUCAP = 3.
\end{corollary}

\begin{proof}
\label{proofcorollary2}
The proof is basically the same as for UCAP = 1 by using the graphs in Figure \ref{fig:generalbinslots} and Figure \ref{fig:generalbin} and thus lifting Lemma \ref{lem2} as well as Theorem \ref{theorem2} to any UCAP $\geq 2$. The corresponding proofs can be lifted by following the structure in Proof \ref{prooflem2} and Proof \ref{prooftheorem2} by simultaneously replacing the indicators/sensors by the corresponding sets of indicators/sensors.$\square$
\end{proof}

\begin{figure}
\center
	\includegraphics[width=13cm]{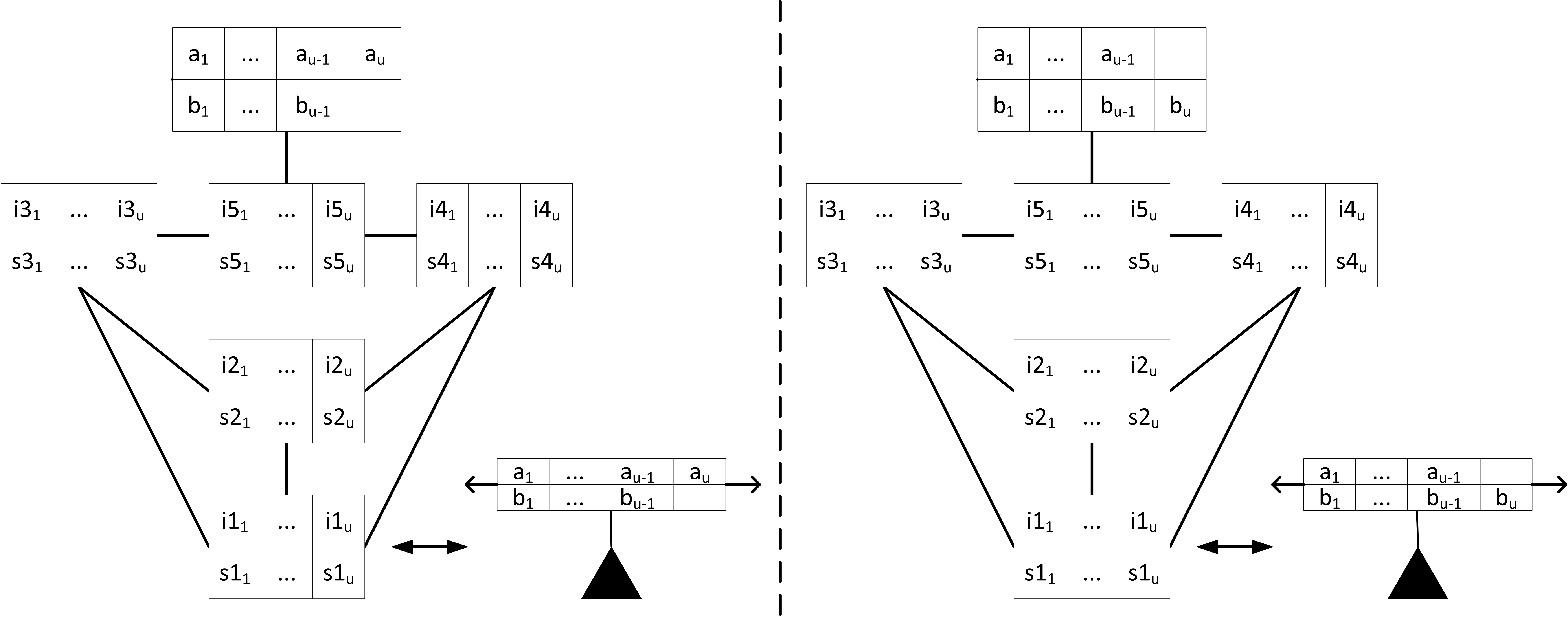}
	\caption{Bins for the PUP with UCAP = u $\geq$ 2}
	\label{fig:generalbinslots}
\end{figure}

\begin{figure}
\center
	\includegraphics[width=13cm]{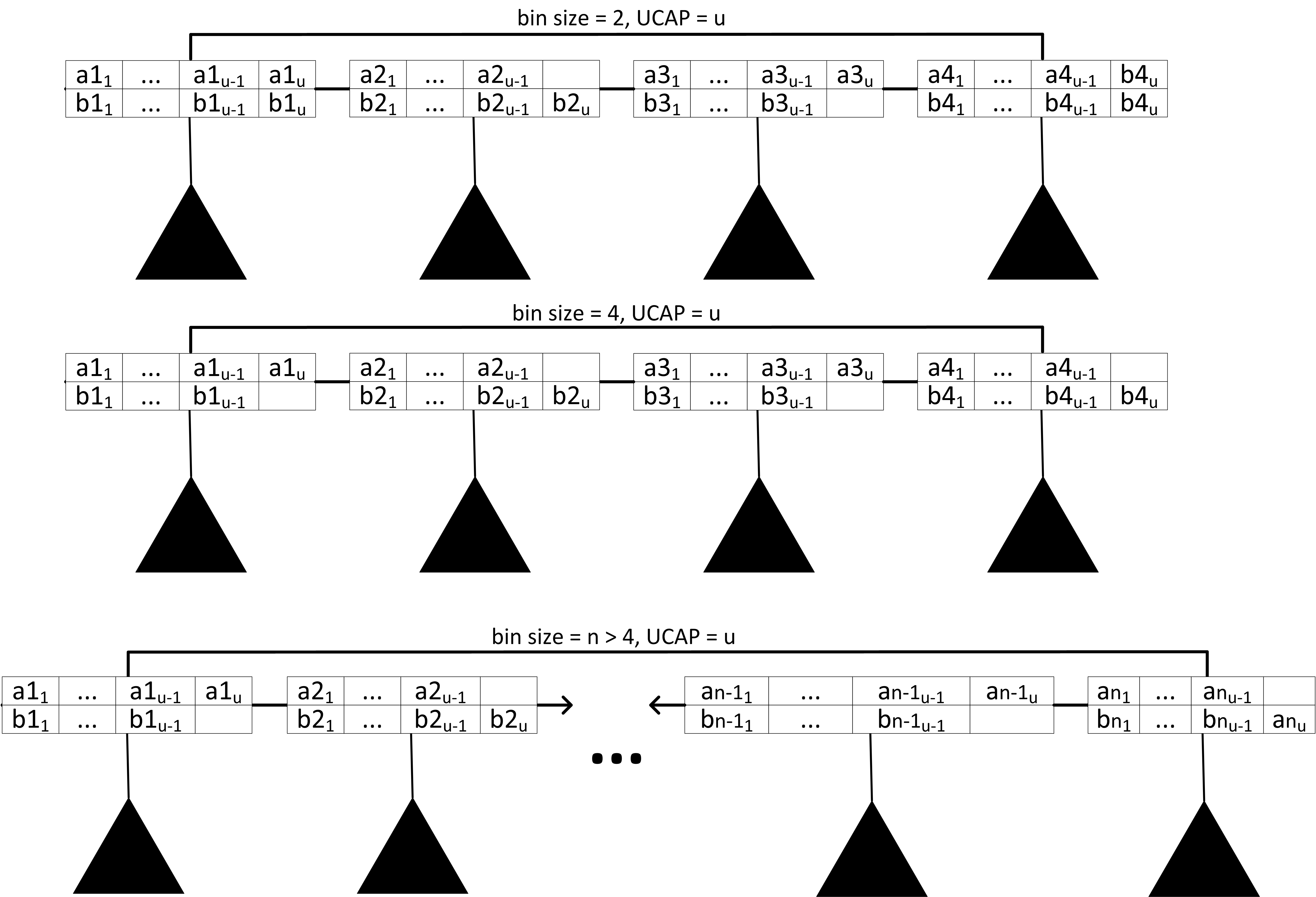}
	\caption{Simulating bins for PUP with UCAP = u $\geq$ 2}
	\label{fig:generalbin}
\end{figure}

For the generalization to any fixed IUCAP $>$ 3 we have to define how the basic graphs in Figure \ref{fig:binslots} respectively Figure \ref{fig:generalbinslots} are evolved in order to 'deactivate' the additional unit connections which come along with a higher IUCAP.

\begin{corollary}
The PUP is NP-complete when UCAP fixed to some $m \geq 2$ and IUCAP is fixed to some $n \geq$ 3.
\end{corollary}

For the ease of readability we explain the proof with respect to Proof \ref{prooflem2}, Proof \ref{prooftheorem2} and Proof \ref{proofcorollary2}.
Basically, the chains of reasoning for all IUCAPs $\geq$ 3 are similar. The main difference are the input graphs to be used for simulating the bins. As the concrete input graphs are becoming to big and confusing we only define their inducing graphs. Note, that by definition the corresponding input graphs are uniquely determined.
Figure \ref{fig:generaliucap} shows the principle for all fixed IUCAPs $\geq$ 3. The nodes are standing for whole units whereby all units except the 'a-untits' respectively the 'b-units' contain UCAP many sensors and indicators. The a-units have one free sensor slot and the b-units have one free indicator slot. The principle becomes obvious in Figure \ref{fig:generalbinslots}. The graphs for IUCAP = 3 are the same as shown in Figure \ref{fig:binslots} and Figure \ref{fig:bin}. 

\begin{proof}

Having the graphs for IUCAP = 3 as the base, For IUCAP = 4 and IUCAP = 5 the edges slightly change in order to adapt to the higher number of available unit connections and to make sure that there are exactly two left over unit connections for the units containing the $a$ and $b$ elements. 

We begin with the fact that the maximum number of elements (i.e. sensors respectively indicators) some indicator respectively sensor can be connected to is $(IUCAP+1)*UCAP$ (\cite{aschinger:tackling}).

As a direct consequence for IUCAP = 4, all elements of $U1$, $U2$, $U3$, $U4$ and $U5$ in the inducing graph must be placed on exactly 5 units in any solution graph leaving no space for any other elements, as all elements in $U1$, $U2$, $U5$ are connected to exactly UCAP many elements in each unit in ${U1, U2, U3, U4, U5}$.
Let us call these units again $U1-U5$ in order to support the close relation between inducing graphs and solution graphs.
Furthermore, the elements of $\{U1,U2,U5\}$ in the inducing graph are stitched to $\{U1,U2,U5\}$ in the solution graph although the elements can be placed on any unit in $\{U1,U2,U5\}$, similar to the elements in $U1$ and $U2$ in the IUCAP = 3 case (compare Proof \ref{prooflem2}). Placing any of these elements in $U3$ or $U4$ would afford that $U3$ and $U4$ are connected, making further connections impossible. Thus, also the position of the elements of $U3$ and $U4$ in the inducing graph are bound to $U3$ and $U4$ in the solution graph.
Moreover, the elements of $U3$ and $U4$ in the inducing graph cannot be mixed up between $U3$ and $U4$ in the solution graph as this would again afford the connection of $U3$ and $U4$, making further connections impossible. As a consequence all elements of the a-units respectively the b-units in the inducing graph must be placed together on the same unit in the solution graph which is connected to $U3$ and $U4$ and thus there are only two unit connections left over. The rest stays the same as in Proof \ref{prooftheorem2} and Proof \ref{proofcorollary2}.

Beginning with IUCAP = 5, it is necessary to consider two graphs (one for the graph containing the a-unit and one for the graph containing the b-unit) in combination. That the elements of $U1$ in the inducing graph have to be on $U1$ in the solution graph respectively the elements of $U1'$ in the inducing graph have to be together on $U1'$ in the solution graph arises out of the fact that the corresponding elements are to be connected to the maximum possible number of other elements (i.e. $(IUCAP+1)*UCAP$) which must be placed on exactly on $IUCAP+1=6$ units. From the perspective of the elements of $U1$ in the inducing graph these are $U1'$,$U1$,$U2$,$U3$,$U4$ and $U5$ in the solution graph. From the perspective of the elements of $U1'$ in the inducing graph these are $U1$,$U1'$,$U2'$,$U3'$,$U4'$ and $U5'$ in the solution graph. Placing only a single element of $U1$ respectively $U1'$ on a different unit (i.e. not $U1$ respectively $U1'$) in the solution graph would afford at least one additional unit connection which is not possible. Analogous arguments hold for the $U2$ in combination with $U2'$, $U3$ in combination with $U3'$ and $U4$ in combination with $U4'$. As $U1$,$U1'$,$U2$,$U2'$,$U3$,$U3'$,$U4$ and $U4'$ in the solution graph contain all elements from $U1$,$U1'$,$U2$,$U2'$,$U3$,$U3'$,$U4$ and $U4'$ in the inducing graph there is no free space left on these units. Consequently the elements of $U5$ respectively $U5'$ in the inducing graph have to be together on the same unit in the solution graph (i.e. $U5$ respectively $U5'$). As $U3$, $U4$, and $U5$ in the solution graph each have only exactly one free unit connection left for connecting to further elements, all of these elements, i.e. all elements of the a-unit respectively the b-unit in the inducing graph, must be placed together on a single unit in the solution graph, which is connected to $U3$, $U4$, and $U5$. Again, the rest stays the same as in Proof \ref{prooftheorem2} and Proof \ref{proofcorollary2}.

Starting from IUCAP = 5 the arguments can be adapted to IUCAP = 6 by introducing two additional units (see Figure \ref{fig:generaliucap}), one for the (sub)graph including the a-unit and one for the (sub)graph including the b-unit, filled with fresh elements which are connected to all possible elements (sensors to indicators and vice versa) in the $IUCAP+1=6$ units of the (sub)graph. As a consequence the new elements are pinned to the new units in the solution graph. By this trick, the arguments which hold for IUCAP = 5 also hold for IUCAP = 6. Again, the rest stays the same as in Proof \ref{prooftheorem2} and Proof \ref{proofcorollary2}.

What remains to be shown is that the proof for some IUCAP = $x \geq 6$ can be done on behalf of the proof for IUCAP = $x-1$. Hereby it suffices to show that any increase of the IUCAP by one can be compensated by the addition of two new units. Hence, it is to show that, given a graph for some IUCAP = $k$, the number of additional unit connections when setting $k:=k+1$ is exactly $2*(k+1)$, i.e.\ the number of connections of two fresh units. This can be done easily by induction.

As a base case we can use the graph for $IUCAP=k=5$. When setting $k:=(k+1)=6$ there are exactly $2*6$ additional unit connections. For the inductive step we use the hypothesis that, given a consistent graph for $IUCAP=k$, the number of units in this graph is $2*(k+1)$. As every unit gets an additional connection when increasing $IUCAP=k$ by one there are also exactly $2*(k+1)$ additional connections which is exactly the number of connections of two fresh units with $IUCAP=k+1$.$\square$
\end{proof}

\begin{figure}
\center
	\includegraphics[width=10cm]{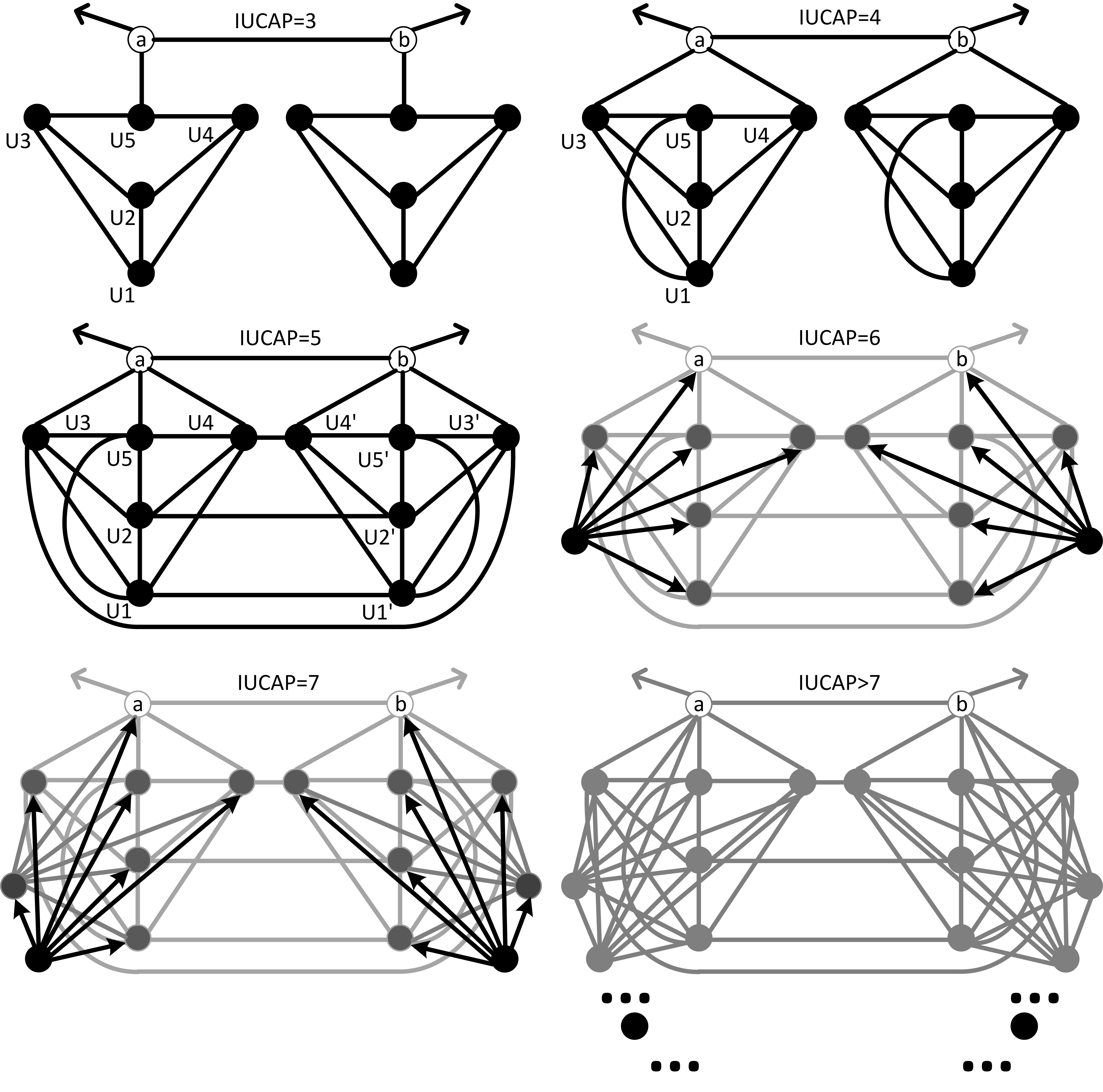}
	\caption{Inducing graphs for any IUCAP $\geq$ 3}
	\label{fig:generaliucap}
\end{figure}

The following consequence follows directly from the presented fixed parameter complexity results.

\begin{corollary}
The PUP is NP-complete when 
\begin{itemize}
\item UCAP is not fixed or
\item IUCAP is not fixed.
\end{itemize}
\end{corollary}

The new complexity landscape for the PUP is shown in Figure \ref{fig:overview}. Hence, complexity of the PUP has been completely clarified.

\begin{figure}
\center
	\includegraphics[width=8cm]{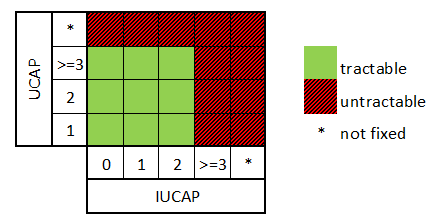}
	\caption{New complexity landscape for the PUP}
	\label{fig:overview}
\end{figure}

\section{QuickPup: A Heuristic Backtracking Search Algorithm}
\label{secQP}

The hardness results of the last section explain why common AI techniques like Constraint Programming, Answer Set Programming, or Integer Programming are often not applicable on real world problem instances with hundreds or thousands of indicators and sensors. Interestingly, human experts in real world PUP domains are often able to produce solutions easily even for big PUP instances where common AI techniques fail. Obviously, human experts use highly efficient heuristics to master this challenge. Out of this, efforts have been made for developing some heuristic algorithm in order to tackle real world sized problems.
 
One result of these efforts is QuickPup (QP), a novel heuristic algorithm for tackling the PUP\footnote{QuickPup has been first introduced in \cite{teppan:quickpup}.}. QP basically follows a backtracking search approach but combines it with a static heuristic ordering of the indicators and sensors (elements). Based on this fixed ordering, QP tries to assign each element to a unit and backtracks in case of unsatisfiability.

\begin{algorithm}\caption{QuickPup: Main}\label{alg:main}\small
\begin{algorithmic}
\STATE INPUT: indicators, sensors, edges, ucap, iucap, maxTime, maxUnits
\STATE
\STATE	timeslice $\gets$ maxTime DIV numberOf(indicators)
\STATE
	\FORALL{startIndicator IN indicators}
\STATE		elements $\gets$ GetBreadthFirstOrder(startIndicator, indicators, sensors, edges)
\STATE 		index $\gets$ firstIndexOf(elements)
\STATE		model $\gets$ \{\} \%\%model is a global variable
\STATE		stopTime $\gets$ SystemTime + timeslice
\STATE		status $\gets$ Assign(elements, edges, ucap, iucap, model, index, stopTime, maxUnits)
\STATE		
		\IF{status = TRUE}
\STATE		Minimize(model)
\RETURN model
		\ELSIF{status = FALSE}
\RETURN FALSE
		\ELSIF{status = TIMEOUT}
\STATE			Continue
		\ENDIF	
		\ENDFOR
\RETURN TIMEOUT
\end{algorithmic}
\end{algorithm}

Algorithm \ref{alg:main} depicts the main procedure of QP. The input consists of a set of indicators, a set of sensors, the edges of the corresponding bipartite input graph, the unit and inter unit capacities and a maximal time limit (maxTime) for solution calculation, and a maximal number of units to be used (maxunits). The first important extension to simple backtracking is to restart the backtracking process from a different entry point if no solution can be found within a certain time slice. If unsatisfiability is proven, no further enrtry points are investigated. In QP each indicator constitutes a possible entry point (startIndicator). For each entry point there is a maximal timeslice of \textit{maxTime DIV number of indicators}. Furthermore, the algorithm produces a different breadth-first ordering of the indicators and sensors (elements) for each entry point. Note that if a concrete implementation of QP is multi-threaded, maxTime and the time slices are not needed, as the algorithm may start from each entry point concurrently.

\begin{figure}
\center
	\includegraphics[width=8cm]{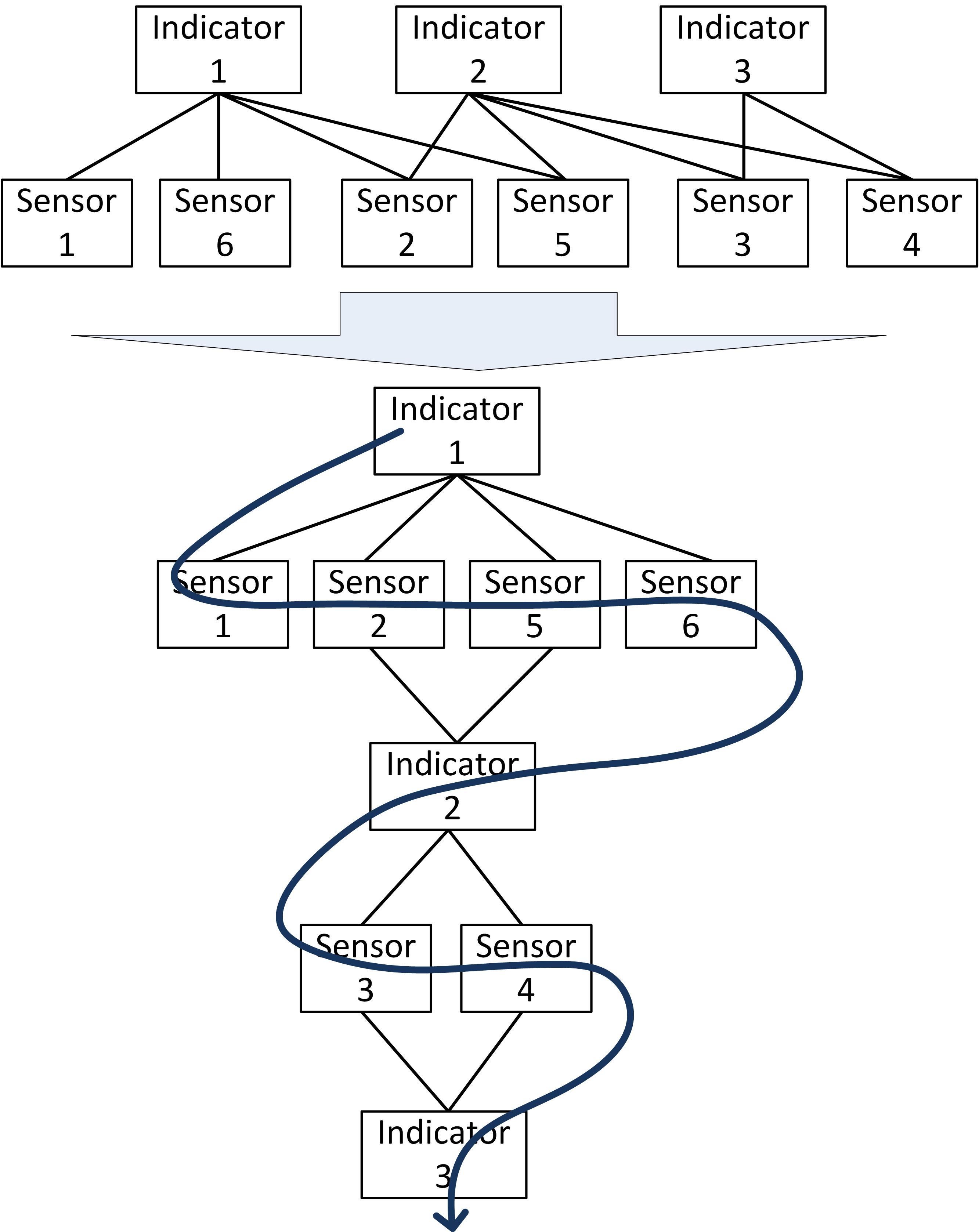}
	\caption{Reordered input graph and corresponding ordering for startIndicator 'A' for example given in Figures \ref{fig:example1}}\label{fig:order}
\end{figure}

For ordering the elements, QP uses a breadth-first strategy (see Figure \ref{fig:order}). Starting from a certain indicator (startIndicator) the next elements to be considered are all connected sensors based on the given input graph. Then, all indicators connected to these sensors are considered, and so forth, until there are no more elements. This way of traversing a graph (i.e.\ the input graph) is known as breadth-first or also as topological order, as the graph is traversed from level to level. Algorithm \ref{alg:order} shows how this is realized. The $\oplus_{l}$ operation stands for inserting an element into a vector of elements at the last position.

\begin{algorithm}\caption{QuickPup: GetBreadthFirstOrder}\label{alg:order}\small
\begin{algorithmic}
\STATE INPUT: startIndicator, indicators, sensors, edges 
\STATE
\STATE elements $\gets$ \{startIndicator\}
\WHILE{sizeOf(elements) $<$ sizeOf(indicators) + sizeOf(sensors)}
\STATE connectedElems $\gets$ getConnectedElems(elements, indicators, sensors, edges)
\FORALL{elem IN connectedElems}
\IF{elem NOT IN elements}
\STATE elements $\gets$ elements $\oplus_{l}$ elem
\ENDIF
\ENDFOR
\ENDWHILE
\RETURN elements
	\end{algorithmic}
\end{algorithm}

Once an element ordering is fixed, QP creates an empty model and calls a recursive sub-procedure ($Assign$, see Algorithm \ref{alg:assign}) creating and connecting the units and trying to assign the elements. Thus, the model (also called the \textit{solution graph} or simply \textit{solution}) consists of the units, their partner unit connections, and the element assignments to the units. 
If $Assign$ runs into a timeout (SystemTime $>$ stopTime), QP continues with the next entry point (i.e.\ startIndicator is reassigned). If all iterations produced timeouts the maxTime is reached and QP stops with no decision. If $Assign$ can prove the unsatisfiability of the given input graph QP returns FALSE. Please note that the combination of multiple start indicators and breadth-first ordering focuses on the early detection of unsatisfiable instances. The idea is that if an instance is unsatisfiable then there is also at least one indicator which is part of the conflict. Iterating through all indicators guarantees that the subsequent backtracking procedure encounters the conflict in the beginning at least once. 

\begin{algorithm}\caption{QuickPup: Assign}\label{alg:assign}\small
\begin{algorithmic}
\STATE INPUT: elements[], edges, ucap, iucap, model, index, stopTime, maxUnits
\STATE
	\IF{index $>$ lastIndexOf(elements)}
\RETURN TRUE
	\ELSIF{SystemTime $>$ stopTime}
\RETURN TIMEOUT
	\ENDIF
\STATE
\STATE	currElem $\gets$ elements[index]
\STATE	
\IF {numUnits(model) $<$ maxUnits}
\STATE	unit $\gets$ createNewUnit(model, ucap, iucap)
	\IF{AssignAndConnect(currElem, unit, model, edges) = TRUE}
\STATE		consistent $\gets$ Assign(elements, model, index + 1, stopTime)
\STATE
		\IF{consistent = TRUE}
\RETURN  TRUE
		\ELSIF{consistent = FALSE}
\STATE			UndoAssignAndConnect(currElem, unit, model, edges)
\STATE			remove(unit,model)
		\ELSIF{consistent = TIMEOUT}
\RETURN TIMEOUT
		\ENDIF
	\ENDIF
	\ENDIF
	\STATE
	\FORALL{oldUnit IN model}
		\IF{AssignAndConnect(currElem, oldUnit, model, edges) = TRUE}
\STATE			consistent$\gets$ Assign(elements, model, index + 1, stopTime)
			\IF{consistent = TRUE}
\RETURN  TRUE
			\ELSIF{consistent = FALSE}
\STATE				UndoAssignAndConnect(currElem, oldUnit, model, edges) 					
\ELSIF{consistent = TIMEOUT}
\RETURN TIMEOUT
			\ENDIF
		\ENDIF
	\ENDFOR
\RETURN FALSE
\end{algorithmic}
\end{algorithm}

If $Assign$ is successful, i.e.\ a consistent assignment for all elements has been found, such that all edges in the input graph are supported, QP minimizes the model and returns it. Minimizing the model in this context means merging units when possible. This step is important for reducing the number of units in the model. Algorithm \ref{alg:minimize} depicts the idea. For pairs of units in the (consistent) model, merging is executed if possible. The $\oplus_{m}$ operator stands for unit merging. If merging is successful, the obsolete unit will be removed from the model.

Actual model checking by backtracking is done by $Assign$. Algorithm \ref{alg:assign} shows the procedure. The input consists of the (ordered) elements, the edges of input graph, the intermediate model, an index pointing to the next element to be assigned, and a time limit (stopTime). First, $Assign$ checks whether the index is greater than the last possible index. In this case all elements have already been assigned successfully and $Assign$ returns TRUE. If this is not the case, $Assign$ checks whether there is still some time left for further calculations, otherwise $Assign$ returns TIMEOUT.

If there is at least one element and some time left $Assign$ proceeds with the assignment of the next element (currElem). To this end QP first creates a new unit of the model and checks whether the assignment to the new unit leads to a consistent intermediate model, i.e.\ all relevant partner unit connections can be established. Please note that a unit is limited in its maximal number of indicators/sensors (UCAP) and its maximal number of partner unit connections (IUCAP).

Consistency checking, the establishment of new partner unit connections and element assignment are carried out in $AssignAndConnect$, see Algorithm \ref{alg:place}. Basically, $AssignAndConnect$ checks two preconditions before an element is assigned to a unit. First, there must be at least one free place left on the unit for picking up a further indicator or sensor, respectively. In the case of a new unit, this precondition is always given. Second, $AssignAndConnect$ verifies that all additional partner unit connections can be established, this being limited by means of IUCAP\footnote{Note, that the partner unit connections are uniquely determined, i.e.\ no search needed.}.

\begin{algorithm}\caption{QuickPup: AssignAndConnect}\label{alg:place}\small
\begin{algorithmic}
\STATE INPUT: currElem, unit, model, edges
\STATE
	\IF{hasFreePlace(unit, currElem) = FALSE}     
\RETURN FALSE
	\ELSIF{relevantUnitsCanBeConnected(currElem, unit, model, edges) = FALSE} 
\RETURN FALSE
	\ELSE
\STATE		add(currElem, unit)
\STATE		establishConnections(currElem, unit, model)
\RETURN TRUE
	\ENDIF
	\end{algorithmic}
\end{algorithm}

If the assignment is successful, $Assign$ calls itself recursively with the updated intermediate model and incremented index pointing to the next element. In case the subsequent $Assign$ returns TRUE, all remaining elements have been assigned consistently, and the current instance of $Assign$ also returns TRUE. If a timeout has been triggered, and hence the return value of the called $Assign$ instance is TIMEOUT, the current $Assign$ back-propagates TIMEOUT.

If the called $Assign$ instance returns FALSE, this means that no assignment for the remaining elements could be found which is consistent with assignment of the current element (currElem) to the newly created unit. In this case, all changes which have been done by $AssignAndConnect$ are revoked and the new unit is removed from the model.

In a second step, QP tries to assign currElem to one of the old units already existing in the model. The procedure for any old unit is similar to the case where new units are exploited, except that it is well possible that the unit could be 'full', i.e.\ there is no free place for the current element on that unit. If no consistent assignment could be found for both, the old units and a newly generated unit, $Assign$ returns FALSE (i.e.\ backtracks).

\begin{algorithm}\caption{QuickPup: Minimize}\label{alg:minimize}\small
\begin{algorithmic}
\STATE INPUT: model
\STATE
\FORALL{unitA IN model AND unitB IN model AND unitA $\neq$ unitB}
			 \IF{NOT tooManyIndicators(unitA $\oplus_{m}$ unitB)}
		\IF{NOT tooManySensors(unitA $\oplus_{m}$ unitB)}
		 \IF{NOT tooManyPartners(unitA $\oplus_{m}$ unitB)}
		\STATE unitA $\gets$ unitA $\oplus_{m}$ unitB
		\STATE remove(unitB, model)
		\ENDIF
		\ENDIF			
	\ENDIF
\ENDFOR
	\end{algorithmic}
\end{algorithm}

It is obvious, that preferring the creation of new units typically results in non-minimal models, regarding the number of units. For optimizing the model QP uses the greedy procedure depicted in Algorithm \ref{alg:minimize}. If the problem is only to decide whether for a given input graph a configuration exists, the optimization step can be skipped.

\section{Experimental Results}\label{results}

In order to test the performance of QP we refer to the results and the benchmark instances presented in \cite{aschinger:opt}. Thereby, the first priority was to come up with a consistent solution respecting UCAP and IUCAP and only as a second priority to come up with a solution minimizing the number of used units. All experiments in \cite{aschinger:opt} were carried out on a 3 GHz dual-core system with 4 GByte of RAM, running Fedora Linux. The results for QP (with model minimizing) were produced by an Intel Core i7 quad-core notebook with 2.8 GHz and 8 GByte of RAM, running Windows 7. QP was implemented in Java 1.6. Thereby, a multi-threaded version of the algorithms was produced. The main idea was to concurrently start one thread per start node (indicator), thus making time slices obsolete. As soon as one of the threads encounters satisfiability or unsatisfiability the procedure stops. Moreover, the maximal number of units was not set such that optimization was purely done by the $Minimize$ procedure. 

In \cite{aschinger:opt} five different implementations were tested\footnote{Detailed information about the implementations can be found in \cite{aschinger:opt}}:

\begin{itemize}
\item DecidePup \\(DP, polynomial time algorithm only for IUCAP = 2 \cite{aschinger:tackling})
\item Constraint Programming \\(CP, implementation with Eclipse-Prolog (www.eclipseclp.org))
\item SAT Solving \\(SAT, MiniSat (www.minisat.se))
\item Integer Programming \\(IP, two different systems were tested: CBC from the COIN-OR project (www.coin-or.org) and IBM's Cplex (www.ibm.com))
\item Answer Set Programming \\(ASP, Clingo from the Potsdam Answer Set Solving Collection \\(potassco.sourceforge.net))
\end{itemize}

The benchmark\footnote{Benchmark instances can be downloaded at http://demo2-iwas.uni-klu.ac.at/pupsolver/} consists of two parts. In part one the corresponding instances are to be solved with a unit capacity of 2 (UCAP=2) and an inter unit capacity of 2 (IUCAP=2). The instances of part two are to be solved with the same UCAP but with an IUCAP = 4. There are four different types of instances: double (dbl) double-variant (dblv), triple (tri), and grid\footnote{Instances \textit{grid-90, ..., grid-99} were removed from the benchmark as the corresponding input graphs were not connected such that those instances can be seen as a collection of trivial non-relevant instances.}. The instances differ in their number of zones and sensors and the number of sensors per zone. Furthermore, the instances have different structural characteristics, as they are patterned on real problem instances. Figure \ref{fig:instances} shows an example of a real world input graph and the grid8-instance of the benchmark\footnote{More details about the instance structure can be found in \cite{aschinger:opt}.}.

\begin{figure}
\center
                \includegraphics[width=14cm]{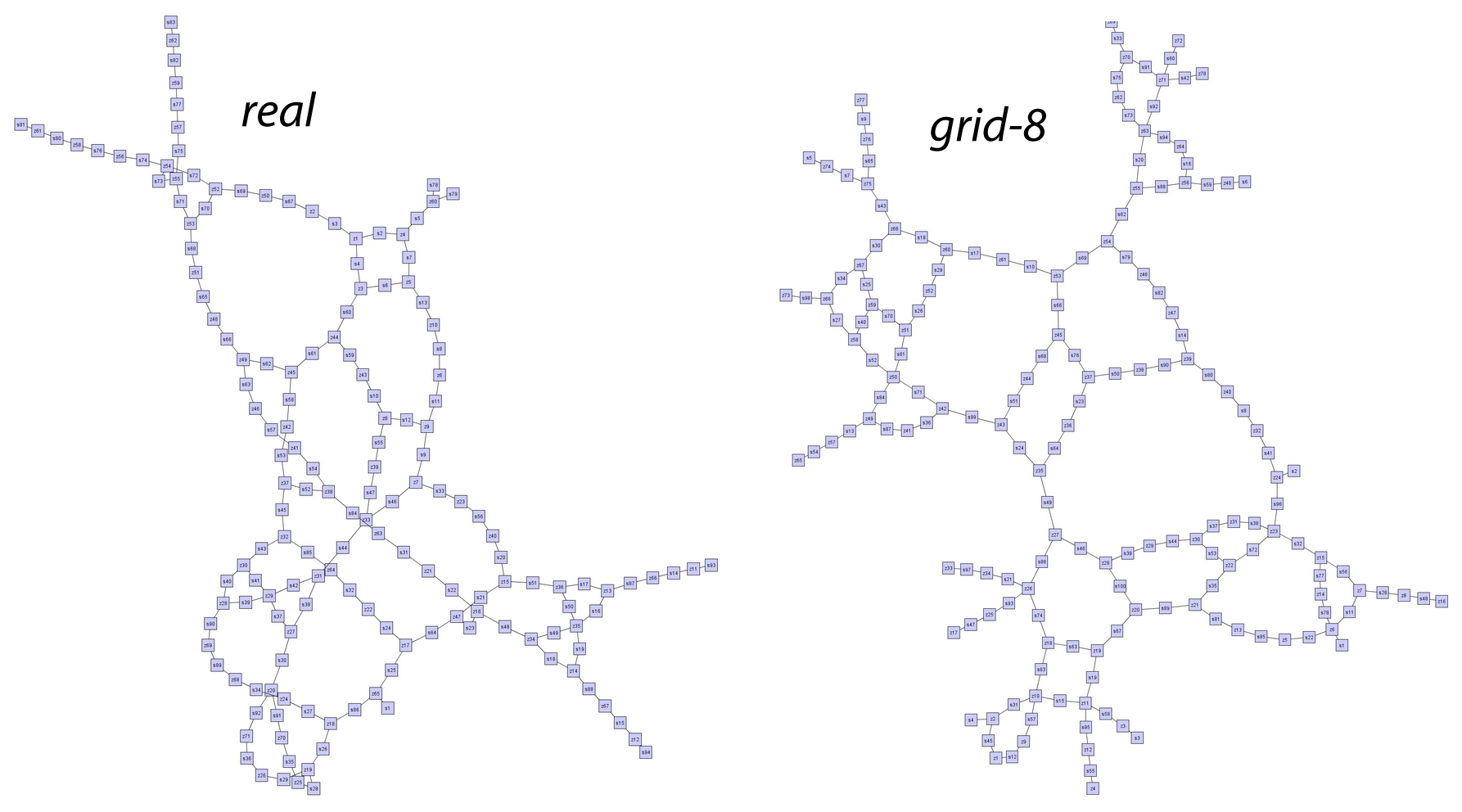}
                \caption{Real problem instance of Siemens Austria in the railway station domain and benchmark instance \textit{grid8} \cite{teppan:quickpup}}\label{fig:instances}
\end{figure}

Table \ref{tab:results2} and Table \ref{tab:results4} summarize the experimental runtimes (seconds) on the described benchmark instances\footnote{For IP the better result produced by the two different approaches is listed.}. The \textit{units} column lists the minimal number of units required for a consistent solution. A minimal number of '0' means that no solution exists (instances tri-34 and tri-64 for IUCAP = 2). In these cases, the results refer to the time needed in order to prove unsatisfiability. A '/' means that the corresponding approach could not solve an instance within a certain time frame. For the experiments in \cite{aschinger:opt} the time frame was limited to 600 seconds. Except for QP, all approaches produced only minimal solutions by construction. The number of additional units needed by QP is listed in the \textit{+units} column.

\begin{table}[htb]
	\centering
		\begin{tabular}{|@{}c@{}|@{}c@{}|@{}c@{}|@{}c@{}|@{}c@{}|@{}c@{}|@{}c@{}|@{}c@{}|@{}c@{}|}
		\hline
			\textsc{Instance}	& \textsc{min \#Units}	& \textsc{DP} & \textsc{SAT} & \textsc{CP}& \textsc{ASP}	& \textsc{IP} & \textsc{QP} &\textsc{+Units}\\
			\hline
dbl-20&14&0.01&0.48&0.02&0.16&1.53&$<$0.01&0\\
dbl-40&29&0.05&2.36&0.28&3.93&13.58&$<$0.01&0\\
dbl-60&44&0.08&29.74&0.42&/&213.58&$<$0.01&0\\
dbl-80&59&0.16&/&1.14&/&522.5&$<$0.01&0\\
dbl-100&74&0.41&/&1.89&/&/&0.01&0\\
dbl-120&89&0.39&/&3.21&/&/&$<$0.01&0\\
dbl-140&104&0.59&/&5.01&/&/&$<$0.01&0\\
dbl-160&119&0.71&/&13.94&/&/&$<$0.01&0\\
dbl-180&134&0.87&/&20.07&/&/&$<$0.01&0\\
dbl-200&149&1.08&/&14.40&/&/&$<$0.01&0\\
\hline
dblv-30&15&65.49&0.42&0.09&0.26&2.93&$<$0.01&0\\
dblv-60&30&/&3.15&0.26&1.94&/&$<$0.01&0\\
dblv-90&45&/&12.54&0.82&27.35&/&$<$0.01&0\\
dblv-120&60&/&41.65&1.85&13.92&/&$<$0.01&0\\
dblv-150&75&/&20.97&3.48&29.54&/&$<$0.01&0\\
dblv-180&90&/&44.28&6.20&54.50&/&$<$0.01&0\\
\hline
tri-30&20&0.50&0.79&1.07&0.41&45.17&$<$0.01&0\\
tri-32&20&/&0.74&0.64&0.26&4.66&$<$0.01&0\\
tri-34&0&/&22.77&21.10&0.89&5.06&$<$0.01&0\\
tri-60&40&114.08&315.42&158.49&4.40&108.01&$<$0.01&0\\
tri-64&0&/&379.36&/&43.88&76.26&$<$0.01&0\\
					\hline
		\end{tabular}
		\caption{Results IUCAP=2, time is given in secs}\label{tab:results2}
\end{table}

Only QP was able to solve all instances. Even DP, which is a polynomial time algorithm capable of only solving problem instances with IUCAP = 2, was not able to solve all instances (with IUCAP = 2). In the cases where the other approaches were able to calculate a solution, QP was always much faster. In fact, the time needed for the calculation of all solutions was significantly below one second. The overhead of additional units used by QP is very small in most cases. As a matter of fact, for almost all instances QP produced minimal solutions, i.e.\ \textit{+units = 0}. Compared to the optimal solution, only for $tri-90$ an additional unit was needed. This makes up a practically negligible increase of 5\%. In \cite{drescher:pup} some concepts of QP were transferred to CP technology and also partially extended. Although, this significantly boosted solution calculation, original QP remained the best approach. 

\begin{table}[htb]
	\centering
		\begin{tabular}{|@{}c@{}|@{}c@{}|@{}c@{}|@{}c@{}|@{}c@{}|@{}c@{}|@{}c@{}|@{}c@{}|@{}c@{}|}
		\hline
			\textsc{Instance}	& \textsc{min \#Units} & \textsc{SAT} & \textsc{CP}& \textsc{ASP}	& \textsc{IP} & \textsc{QP} &\textsc{+Units}\\
			\hline
tri-30&20&2.40&0.12&0.40&24.79&$<$0.01&0\\
tri-32&20&1.91&0.14&0.66&20.84&$<$0.01&0\\
tri-34&20&1.98&/&0.60&/&$<$0.01&1\\
tri-60&40&/&0.52&11.07&/&$<$0.01&0\\
tri-64&40&/&/&7.61&/&$<$0.01&0\\
tri-90&59&401.44&1.50&332.34&/&$<$0.01&0\\
tri-120&79&/&3.37&/&/&$<$0.01&0\\
\hline
grid1&50&78.19&/&31.45&/&$<$0.01&0\\
grid2&50&90.89&/&18.91&/&$<$0.01&0\\
grid3&50&88.87&/&25.72&/&$<$0.01&0\\
grid4&50&95.12&/&24.66&/&$<$0.01&0\\
grid5&50&454.42&/&48.88&/&$<$0.01&0\\
grid6&50&204.85&/&9.15&/&$<$0.01&0\\
grid7&50&112.36&/&12.89&/&$<$0.01&0\\
grid8&50&/&/&11.89&/&$<$0.01&0\\
grid9&50&91.62&/&19.71&/&$<$0.01&0\\
grid10&50&545.16&/&13.54&/&$<$0.01&0\\
					\hline
		\end{tabular}
		\caption{Results IUCAP=4, time is given in secs}\label{tab:results4}
\end{table}

\newpage
\section{Conclusions}

The partner units problem (PUP) is an important problem in the domain of knowledge-based configuration and furthermore acknowledged as hard benchmark problem for Logic Programming, Answer Set Programming and Constraint Programming. There are various application domains for the PUP such as railway safety, CCTV surveillance or electrical engineering.

Although there has been remarkable effort in investigating the problem, the complexity remained widely unclear. This article closes the gap by summarizing already existing complexity results and providing NP-completeness results for all problem subclasses of which the exact complexity class was unknown so far.

Furthermore, we present the QuickPup algorithm which is a heuristic backtracking search algorithm for the PUP. The comparison of new runtime results with results presented in (\cite{aschinger:opt}) on benchmark instances patterned on real life problem instances shows the clear superiority of QuickPup. Since QuickPup is currently the best known algorithm for the PUP, this problem solving strategy was integrated in the Siemens' Configuration Problem Solving Engine (\cite{supreme}) and has already been successfully applied for real world configuration and reconfiguration problems (\cite{teppan:reconf}).

\bibliographystyle{ACM-Reference-Format-Journals}
\bibliography{acmsmall-sample-bibfile}

\end{document}